\def\year{2018}\relax
\documentclass[letterpaper]{article} 
\usepackage{aaai18}  
\usepackage{times}  
\usepackage{helvet}  
\usepackage{courier}  
\usepackage{url}  
\usepackage{graphicx}  
\usepackage{subcaption}
\usepackage{hyperref}
\usepackage{amsfonts,amsmath,amsthm}
\usepackage{graphicx}
\usepackage{color}
\usepackage{algorithm, algorithmic}
\usepackage[mathscr]{euscript}
\usepackage{multirow}

\usepackage{bbm} 

\newcommand{\p}{p}  
\renewcommand{\S}{\mathscr{S}} 
\newcommand{\A}{\mathscr{A}} 
\renewcommand{\r}{r} 
\newcommand{\M}{\mathscr{M}} 
\newcommand{\Q}{q}
\newcommand{\q}{q}

\newcommand{\T}{\mathcal{T}}
\newcommand{\R}{\mathcal{R}}
\renewcommand{\P}{\mathcal{P}}
\newcommand{\Ppi}{\P^\pi}

\renewcommand{\O}{\mathscr{O}} 
\newcommand{\I}{\mathscr{I}} 
\newcommand{\pio}{{\pi^o}} 
\newcommand{\bo}{\zeta^o}
\newcommand{\io}{\I^o}
\newcommand{\Pom}{P}
\newcommand{\Ro}{R}
\newcommand{\To}{\T_{\O}}

\newcommand{\E}{\mathbb{E}} 
\newcommand{\Exp}[2]{\E_{#1}\left[#2\right]}
\newcommand{\Real}{\mathbb{R}}
\newcommand{\bN}{\mathbb{N}}
\newcommand{\1}{\mathbb{I} } 

\newcommand\defequal{\mathrel{\overset{\makebox[0pt]{\mbox{\tiny def}}}{=}}}

\newtheorem{theorem}{Theorem}

\newtheorem{proposition}{Proposition}
\newtheorem{corollary}{Corollary}

\newif\iflong
\longfalse

\begin{document}
%
\title{Learning with Options that Terminate Off-Policy}

\author{Anna Harutyunyan \\ 
 Vrije Universiteit Brussel \\ 
 Brussels, Belgium 
 \And Peter Vrancx\\
 PROWLER.io\\
 Cambridge, England\\
\And Pierre-Luc Bacon \\ 
 McGill University\\
 Montreal, Canada\\
\And Doina Precup\\
McGill University\\
 Montreal, Canada\\
\And Ann Now\'{e} \\
Vrije Universiteit Brussel \\ 
 Brussels, Belgium 
 }

\maketitle
\begin{abstract}
  A temporally abstract action, or an {\em option}, is specified by a
 policy and a termination condition: the policy guides the option behavior,
 and the termination condition roughly determines its length. 
Generally, learning with longer options (like learning with multi-step returns)
is known to be more efficient. However, if the option set for the task is not
 ideal, and cannot express the primitive optimal policy well,
 shorter options offer more flexibility and can yield a better
 solution. 
Thus, the termination condition puts learning efficiency at odds with solution
 quality.
We propose to resolve this dilemma by decoupling the {\em behavior} and
 {\em target} terminations, just like it is done with policies in
 off-policy learning. To this end, we give a new algorithm, Q($\beta$),
 that learns the solution with respect to 
 {\em any} termination condition, regardless of how the options actually terminate.
 We derive Q($\beta$) by casting learning with options into a
   common framework with well-studied multi-step off-policy 
     learning. 
We validate our algorithm empirically, and show that it holds up to its motivating claims. 
\end{abstract}

\section{Introduction}

Abstraction is essential for scaling up learning, and
there has been a renewed interest in methods that extract, or leverage
it~\cite{vezhnevets2016strategic,kulkarni2016hierarchical,tessler2017deep}.
The options
framework~\cite{sutton1999between} is the 
standard for
modeling {\em temporal} abstraction in reinforcement learning. 
The temporal aspect of an option is determined by its {\em
  termination condition} $\beta$, which roughly determines its length.
Learning and planning with longer
options is known to be more efficient~\cite{mann2015approximate}. This
is partly due to an option having similar properties
to the familiar multi-step $\lambda$-returns,\footnote{The
  similarity is particularly relevant since a full  $\beta$-model~\cite{sutton1995td} is at the basis of both paradigms.} which are known to yield faster convergence~\cite{bertsekas1996neurodynamic}.
The key qualitative difference between $\beta$ and $\lambda$, however, is that
$\beta$ directly affects the solution rather than, like $\lambda$, just the rate of
convergence. If $\beta$ is not trivial, this couples the
quality of the solution with the quality of the options at hand. This
can be restrictive especially if
the options are not perfect, which of course is likely.
Indeed, when a set of options is given, we can show that the more
they terminate, the more optimal the resulting policy is at the
primitive action level. This poses a challenge: on the one hand, we
wish for the options to be long to yield fast convergence and meaningful exploration,
but on the other, if these options are not ideal, the more we commit
to them, the poorer the quality of our solution. {\em Interrupting} suboptimal
options is one way of addressing this~\cite{sutton1999between}, but
like ``cutting'' traces in off-policy learning, it may prevent us
from following a coherent policy for more than a couple of steps.

To this end, we propose to terminate options {\em off-policy}, that is:
decouple the {\em behavior} termination condition that the options
execute with, from the {\em target} termination condition that is
to be factored into the solution. 
The behavior terminations then solely become a means of influencing
convergence speed. We describe a new algorithm, Q($\beta$), that achieves this by
leveraging connections to several old and new multi-step
off-policy temporal difference algorithms.

The paper is organized as follows. After introducing relevant background, we will
derive the fixed point of the classical call-and-return option
operator. 
Using its shape for intuition, we will then incrementally build
up to our proposed {\em off-policy} termination option operator, starting
from the classical intra-option equations. We will analyze the
convergence properties of this operator for both policy evaluation and
control, and state the corresponding online algorithm
Q($\beta$). Finally, we will validate Q($\beta$) empirically, and show
that it can (1) learn to evaluate a task w.r.t. options terminating
off-policy, and (2) learn an optimal solution from suboptimal options quicker than the alternatives.


\section{Framework and notation}

We assume the standard reinforcement learning setting~\cite{sutton-barto17} of an MDP $\M =
(\S, \A, \p, \r, \gamma)$, where $\S$ is the set of states, $\A$ the set
of discrete actions; $\p:\S\times\A\times \S \rightarrow [0,1]$ the
transition that specifies the environment dynamics, with $\p(s'|s, a)$
 (or $\p_{ss'}^a$ in short),
denoting the probability of transitioning to state $s'$ upon taking
action $a$ in $s$; $\r : \S\times\A \rightarrow
[-\r_{\max},\r_{\max}]$ is the reward function, and $\gamma$ the
scalar discount factor. A policy is a probabilistic mapping from
states to actions. For a policy $\pi$, and the MDP transition matrix $\p$, let the matrix $\p^\pi$ denote the
dynamics of the induced Markov chain: $\p^\pi(s,s') =
\sum_{a\in\A}\pi(a|s)\p(s'|s,a),$ and $\r^\pi$ the reward expected for each
state under $\pi$: $r^\pi(s) = \sum_{a\in\A} \pi(a|s) \r(s,a).$

Consider the Q-function $\Q$ as a
mapping $\S \times \A \to \Real$, and define the one-step transition
operator over Q-functions as:
\begin{equation}
(\Ppi \Q)(s,a) \defequal \sum_{s' \in \S} \sum_{a' \in \A} p(s' |
s, a) \pi(a' | s') \Q(s', a'). \label{eq:Ppi-operator}
\end{equation}
Using operator notation, define the Q-function corresponding to the {\em value} of policy $\pi$
as: \[ \q^\pi \defequal \sum_{t=0}^\infty \gamma^t (\Ppi)^t \r = \r +
  \gamma\Ppi\q^\pi = (I-\gamma\P^\pi)^{-1}\r,\] 
the recursive form of which defines the Bellman
equation for the policy $\pi$~\cite{Bellman:1957}. The corresponding
one-step Bellman operator can be applied to any Q-function:
 \begin{align}
\label{eq:bellman-q-operator}
\T^\pi \Q (s,a) & \defequal r(s,a) + \gamma \Ppi \Q(s,a),
\end{align}
and its repeated applications are guaranteed to produce its fixed
point $\q^\pi$~\cite{puterman94}. The
policy evaluation setting is concerned with estimating this quantity for a
given policy $\pi$.

The Bellman {\em optimality} operator introduces maximization over the set
of policies: $\T\Q \defequal r+ \gamma \max_\pi \Ppi \Q,$
and its repeated applications are guaranteed to produce its fixed
point, $\q^* \defequal \max_{\pi}\q^\pi = \q^{\pi^*}$, which is the value of the {\em optimal}
policy $\pi^*$. The control setting is concerned with finding this value. A
policy is {\em greedy} w.r.t. a Q-function if at each state it
picks the action of maximum value. The optimal policy $\pi^*$ is
greedy w.r.t. the optimal Q-function $\q^*$.

Throughout, we will write ${\mathbf 1}$ and ${\mathbf 0}$ for the
vectors with 1- or 0-components. For a policy $\pi^b$, we will use $\E_{\pi^b}[\cdot]$ as a shorthand
for the expectation $\E_{S_{1:\infty}, A_{1:\infty} | \pi^b}[\cdot]$ w.r.t. the
trajectory $S_0=s, A_0=a, S_1, A_1, \ldots$, with $A_t$ drawn according to
$\pi^b (\cdot | S_t)$ and $S_{t+1}$ drawn according to 
$\p(\cdot|S_t, A_t)$.
We will occasionally use $R_{t+1}$ to denote
$r(S_t, A_t)$, and write arguments in subscripts (i.e. use $x(s)$
and  $x_s$ equivalently). In the learning setting, the algorithms often update with a sample
of the residual $\T^\pi\q-\q$:
\[\delta_t = R_{t+1} + \gamma\q(S_{t+1}, A_{t+1}) -\q(S_t,A_t),\]
which is referred to as a temporal difference (TD) error.

\subsection{Multi-step off-policy TD learning}

One need not only consider single-step operators, and may apply
$\T^\pi$ and $\T$ repeatedly. The
$\lambda$-operator is a particularly flexible mixture of such multi-step operators:
\[\T^\pi_\lambda\q =
  (1-\lambda)\sum_{t=0}^\infty\lambda^n(\T^\pi)^n\q = \q +
  (I-\lambda\gamma\P^\pi)^{-1}(\T^\pi\q - \q).\]
In the learning setting, this corresponds to considering multi-step
{\em returns}, rather than one-step samples for updating the estimate, and $\lambda$ trades off the bias of
bootstrapping with an approximate estimate with the variance of
sampling multiple steps. A high intermediate value of $\lambda$ is
typically best in practice~\cite{kearns2000bias}.

{\em Off-policy learning} is the setting where the behavior and target
policies are decoupled. That is: $\pi^b \neq \pi$. Multi-step methods pose a challenge when considered off-policy, and
advances have recently been made in this
direction~\cite{munos2016safe,mahmood2017multistep}. To this end,
\citeauthor{munos2016safe}~\shortcite{munos2016safe} unified several off-policy return-based algorithms under a common umbrella:
\begin{align}
\R^{\pi, \pi^b}_c\Q(s,a) & \defequal \Q(s,a) +  \E_{\pi^b} \Big [\sum_{t=0}^\infty
  \gamma^t\Big (\prod_{i=1}^t c_i \Big) \delta^\pi_t \Big], \label{eq:general-operator} \\
 \delta^\pi_t & =  R_{t+1} + \gamma \E_\pi\Q(S_{t+1},\cdot) - \Q(S_t,  A_t). \notag
\end{align}
where \[\E_\pi\Q(s\cdot) \defequal \sum_{a\in\A}\pi(a|s)\q(s,a),\]
and the {\em trace} $c_i$ is a state-action coefficient, whose
specific shapes correspond to different algorithms. 
In particular, {\em Tree-Backup($\lambda$)} sets $c_i$ to
$\lambda\pi_b(A_i|S_i)$~\cite{precup2000eligibility},
while {\em Retrace($\lambda$)} sets $c_i$
to $\lambda\min(1, \frac{\pi(A_i|S_i)}{\pi_b(A_i|S_i)})$~\cite{munos2016safe}.

\subsubsection{Unifying algorithm}
Recently,
\citeauthor{deasis2017multistep}~\shortcite{deasis2017multistep} formulated an algorithm that
encompasses the ones discussed so far even more generally. Instead of
the binary taxonomy of on-and off-policy algorithms, the authors
introduce a parameter
$\sigma$ that smoothly transitions between the two. This algorithm can also
be expressed via Eq.~\eqref{eq:general-operator} with:
\begin{align*}
  \delta^{\sigma,\pi}_t  & =  R_{t+1} + \gamma (\sigma\q(S_{t+1}, A_{t+1})
  \\
   & \qquad +
  (1-\sigma)\E_\pi\Q(S_{t+1},\cdot)) - \Q(S_t,  A_t), \\
   c_i & = \lambda((1-\sigma)\pi_b(A_i|S_i) + \sigma).
\end{align*}
In particular, $\sigma=1$ corresponds to the on-policy SARSA($\lambda$) algorithm, while
$\sigma=0$ to Tree-Backup($\lambda$).

\subsection{Options}
An option $o$ is a tuple $(\io, \bo, \pio)$, with $\io\subseteq\S$ the
initiation set, from which an option may start, $\bo:\S\rightarrow
[0,1]$, the probabilistic termination condition, and $\pio$,
the option policy with which it navigates through the environment.
Just like the MDP reward and transition models $r$ and $p$, options can be seen to
induce {\em semi-}MDP~\cite{puterman94} models $\Ro$ and $\Pom$ as follows~\cite{sutton1999between}: 
\begin{align}
  \Pom_{ss'}^o & \defequal \Exp{D:s\rightarrow s' | o}{\gamma^D}
                \notag 
\\  & 
= \gamma \p^{\pio}_{ss'} \beta_{s'} + \gamma \sum_{s''\in\io} 
     p^\pio_{ss'} (1-\beta_{s''}) \Pom_{s''s'}^o, \notag 
 \\
\Ro_s^o & \defequal \Exp{D:s | o}{\sum_{i=1}^D
     \gamma^{i-1}r^\pio(S_{t+i}) | S_t = s} \notag \\
& = \r_s^{\pio} +  \gamma \sum_{s'\in\io} 
     p^\pio_{ss'}  (1-\beta_{s'}) \Ro^o_{s'}. \label{eq:Ro-plain}
\end{align}
where $\Exp{D:s | o}{\cdot}$ and $\Exp{D:s\rightarrow s' | o}{\cdot}$  are the expectations of the option
duration $D$ from state $s$ and the travel time between states $s$ and
$s'$, respectively, w.r.t. option dynamics $p^\pio$ and the
termination condition $\beta^o$.

In the {\em call-and-return} model of option execution, an option is
run until completion (according to its termination condition),
and only then a new option choice is
made~\cite{precup1998theoretical}. 
This suggests the following state{\em-option} analogues of the
state-action transition operator from Eq.~\eqref{eq:Ppi-operator}, and
the Bellman operator from Eq.~\eqref{eq:bellman-q-operator}. For a policy over options $\mu$:
\begin{align}
  (\P^\mu_{\O}\q)(s,o) & \defequal \sum_{s'}P^o(s,s') \sum_{o'}\mu(o'|s')
  \q(s',o') \label{eq:Pmu-operator}\\
  \To^\mu \q(s,o) & \defequal \Ro^o_s + \P^{\mu}_{\O}\q(s,o). \label{eq:Bellman-options}
\end{align} 
 It will also be relevant to consider the {\em marginal} flat policy $\kappa$ over primitive actions:
\begin{equation}
  \label{eq:sigma}
  \kappa(a|s) \defequal \sum_o\mu(o|s)\pio(a|s).
\end{equation}
For simplicity, we assume that options can initiate anywhere: $\io = \S,
\forall o\in\O$.
Finally, in the main departure from the standard framework, we will distinguish
between the  {\em target} termination condition $\beta$ that is to be
estimated, and a {\em behavior} termination condition $\zeta$ ({\em
  ``zeta''}) that the options actually terminate with.



\section{The call-and-return operator}


Before proceeding to describe our key idea and algorithm, let us
quantify the role of $\beta$ in the target solution of planning with
options. To do this, we plan to
derive this solution at the primitive action resolution. 

Let $\nu$ be an arbitrary policy over
options, and $c: \S\times \O \rightarrow [0, 1]$ a
coefficient function. Consider the following transition operator and
its corresponding Bellman operator: 
\begin{align}
   (\P^{c\nu}\q)(s, o) & \defequal  \sum_{s'\in\S} \p^\pio_{ss'} c(s',
                         o) \sum_{o'}\nu(o'|s')
                         \q(s',o'), \notag \\
\T^{c\nu}\q & \defequal c r^\pi + \gamma\P^{c\nu}\q, \notag
\end{align}
where $r^\pi$ is the $|\S|\times|\O|$-vector of $r^\pio$ for all options. Note that $\T^{c\nu}$, like
$\T^\pi$, is a one-step operator, whereas  $\To^\mu$ is an
option-level operator.
The transition operator $\P^{c\nu}$ in particular defines the
following operators corresponding to option {\em continuation} and {\em termination}, respectively:
\begin{align*}
  \P^{(1-\beta)\iota}(s, o) & \defequal \sum_{s'\in\S} \p^\pio_{ss'}
                         (1-\beta^o(s'))\q(s',o), \\
 \P^{\beta\mu} (s, o) & \defequal \sum_{s'\in\S} \p^\pio_{ss'}
                         \beta^o(s')\sum_{o'} \mu(o'|s') \q(s',o'). 
\end{align*}
That is: $\iota$ ({\em ``iota''}) is the policy over options that maintains the
current (argument) option. Using these operators, we can
express $\P_{\O}^\mu$ from Eq.~\eqref{eq:Pmu-operator} and the reward
model from Eq.~\eqref{eq:Ro-plain} concisely for
all state-option pairs:
\begin{align*}
  \P_{\O}^\mu\q & = (I - \gamma\P^{(1-\beta)\iota})^{-1}
                  \gamma\P^{\beta\mu}\q, \\
   \Ro & = (I - \gamma\P^{(1-\beta)\iota})^{-1} r^\pi,
\end{align*}
and rewrite the option-level Bellman operator from
Eq.~\eqref{eq:Bellman-options}:
\begin{align}
   \To^\mu\q & 
=  (I - \gamma\P^{(1-\beta)\iota})^{-1} (\r^\pi + \gamma
     \P^{\beta\mu} q). \label{eq:To} 
\end{align}
The following
proposition derives the fixed point of $\To^\mu$ in terms of the one-step
operators $\P^{(1-\beta)\iota}$ and  $\P^{\beta\mu}$.
\begin{proposition}
\label{prop:offpol-options}
  The fixed
 point of the operator $\To^{\mu}$ 
 is the same as the fixed point of the operator
 $\T^{(1-\beta)\iota} + \T^{\beta\mu}$, and writes:  
\begin{equation}
\label{eq:q-mu-iota}
\Q^{\mu,\iota}_\beta =   (I-\gamma(\P^{\beta\mu}-\P^{\beta\iota})
-\gamma\P^{1\iota})^{-1} \r^\pi.
\end{equation}
\end{proposition}
Thus, the termination scheme directly affects the
convergence limit: in the extreme, if $\beta={\mathbf 0}$, options never
terminate, and we have the fixed
point of $\T^{1\iota}$: $\q^{\mu,\iota}_0(s,o) = \E_\pio\q^{\pio}(s, \cdot)$, the value of the option $o$. In the other
extreme, $\beta={\mathbf 1}$, the options terminate at every step and we have
the fixed point of $\T^{1\mu}$, which can be shown (Theorem 1 in \cite{bacon16matrix}) to correspond to the value of the
marginal policy $\kappa$ from Eq.~\eqref{eq:sigma}:
\begin{equation}
\q^{\mu,\iota}_1(s,o) = \E_\kappa\q^{\kappa}(s,\cdot), \forall o\in \O. \label{eq:qkappa}
\end{equation} 

\section{Off-policy option termination}
\label{sec:offpolicy-term}
We would like to decouple the {\em
  target} termination condition $\beta$ that factors into the
solution from the {\em behavior} termination condition
$\zeta$  that governs for how long the options are followed. Apart from the theoretical appeal of the freedom that this
allows, a key motivation is in the fact that on the one hand just like with multi-step
returns, the less options terminate the faster the convergence, but on the other 
the more options terminate, the better the control solution (as we
show formally in the next section). Being able to decouple the two
allows one to achieve the best of both worlds, 
and is exactly what we propose to do in this paper.

The critical insight in our approach is the {\em
  off-policy-ness} at the primitive action level that is introduced by
the discrepancy between policies $\mu$ (that picks a new option) and
$\iota$ (that maintains the current option) in
Eq.~\eqref{eq:q-mu-iota}. 
 The {\em degree} of this off-policy-ness is modulated exactly by the
 termination condition $\beta$, and is usually implicit. 
 We propose to 
leverage multi-step off-policy learning and to impose a desired degree
{\em explicitly}, independent of the behavior terminations $\zeta$.\footnote{Technically this should be
referred to as {\em off-termination} learning, since the difference is indeed
 in termination conditions, of which the induced policies are a
 consequence: the ``behavior policy'' $\iota$ takes the current option w.p. 1, while
 the ``target policy'' is the actual policy over options $\mu$.} 
In the extreme, we can learn the marginal policy
$\kappa$ directly. In the other extreme, we can learn the value of the current
option. The two extremes are traded off via $\beta$. 
The algorithm we propose is analogous to the unifying algorithm Q($\sigma$) in
which $\sigma$ modulates the degree of
off-policy-ness~\cite{deasis2017multistep}.

To highlight the parallels between learning with options and
learning off-policy from multi-step returns, this section will
incrementally build the technical intuition towards the proposed {\em off-policy termination} operator, starting from the classical
intra-option equations. The next section will analyze the convergence
of this operator.

\subsection{From one step intra-option learning to General Q($\lambda$)}
To begin, let us first re-derive the target from
Proposition~\ref{prop:offpol-options} starting from the familiar
intra-option equations~\cite{sutton1998intra}. At this point let us
assume that the options execute w.r.t. some behavior condition $\zeta$.
Letting $\Delta$ denote the
update on the estimated Q-function, we have at time $t$ and the
current option $o$: 
\begin{align}
  \Delta\q(S_t, o) & \propto r(S_t, A_t) + \gamma \tilde{\q}(S_{t+1}, o) - \q(S_t,
  o), \label{eq:intra-option}\\
\tilde{\q}(s, o) & = (1-\bo(s))\q(s,o) + \bo(s)\E_\mu\q(s,\cdot), \notag
\end{align}
where as before we write 
  $\E_{\mu}{\q(s,\cdot)} \defequal \sum_{o}\mu(o|s)\q(s,o).$
Notice that this is exactly a sample of the one-step update
corresponding to $\T^{(1-\zeta)\iota} + \T^{\zeta\mu}$. In fact, if we
roll it out over multiple steps, and take an expectation, we obtain
Eq.~\eqref{eq:To} for the behavior $\zeta$ exactly:
\begin{align}
q(s, o) & = \E_{\pio}\Big[
\sum_{t=0}^\infty 
 \gamma^t \Big(\prod_{i=1}^t (1 -
                   \bo(S_i)) \Big) \notag \\
  & \qquad 
     [ r(S_t, A_t)  + \gamma\bo(S_{t+1})\E_\mu\q(S_{t+1},\cdot)]
    \Big], 
    \label{eq:intra-option-multi-step}
\end{align}
from which some simple algebra yields:
\begin{align}
\Delta q(s, o) & = r^{\pi_o}(s) + \gamma\E_\mu\q(S_{t+1},\cdot) -
                 \q(s, o) \notag 
\\ & \quad +
 \E_{\pio}\Big[ \sum_{t=1}^\infty 
 \gamma^t \Big(\prod_{i=1}^t (1 - \bo(S_i))\Big) \delta^{\mu\mu}_t
     \Big], \label{eq:peng} \\
  \delta^{\mu\mu}_t & = r(S_t, A_t) +
    \gamma\E_\mu\q(S_{t+1},\cdot) - \E_\mu\q(S_{t},\cdot). \notag
 \end{align}
This is the same (expected) update as Peng's Q($\lambda$) for greedy
policies~\cite{peng1996incremental} or General Q($\lambda$)
(i.e. off-policy Expected SARSA($\lambda$)~\cite{van2009theoretical})
for arbitrary policies, with $1 - \bo(S_i)$ being the state-option analogue of
$\lambda$ from those algorithms.
The fixed point of both of those algorithms is given in Proposition 1 in
\cite{harutyunyan16qlambda} and is indeed analogous to that given in Proposition~\ref{prop:offpol-options}
in this paper. 

\subsection{General Q($\lambda$) to Tree-Backup($\lambda$)}
The update~\eqref{eq:peng} converges to the fixed point from
Prop.~\ref{prop:offpol-options}, which is an on-/off-policy mixture,
regulated by the behavior terminations $\zeta$. We wish to reweigh this
mixture by the desired target terminations $\beta$. Let us now begin
introducing the off-policy machinery that will allow us to do so. It
will be useful to cast the multi-step intra-option update~\eqref{eq:peng} in
the form of the general off-policy operator~\eqref{eq:general-operator}. 
This can be done by replacing the second expectation in the TD-error with the point-estimate $q(S_t,
o)$, which introduces {\em off-policy
  corrections}: 
\begin{align*}
 \delta^\mu_t & = r(S_t, A_t) + \gamma\E_\mu\q(S_{t+1},\cdot) - \q(S_{t},o).
\end{align*}
This update will converge to the ``off-policy'' fixed point of $\T^{1\mu}$ only if $\mu$ and $\iota$ are
close~\cite{harutyunyan16qlambda}, which, in our particular case of
the indicator behavior policy $\iota$, is not a very interesting
scenario.
 If we further augment the ``trace'' $1 - \bo(S_i)$ with the policy probability
coefficient $\mu(o|S_i)$, we obtain option-level
Tree-Backup($\lambda$)~\cite{precup2000eligibility} whose target
policy is $\mu$, and behavior policy is $\iota$, and with $1 -
\bo(S_i)$ being the state-option analogue of $\lambda$:
\begin{align}
\Delta q(s, o) & = \E_{\pio}\Big[ \sum_{t=0}^\infty 
 \gamma^t \Big(\prod_{i=1}^t c^o_i\Big) \delta^\mu_t
                 | S_0 = s\Big], \label{eq:tb-full} \\
c^o_i & = \mu(o|S_i)  (1 - \zeta^o(S_i)). \notag 
  \end{align}
From the convergence guarantees of Tree-Backup, we know that this update
converges to the fixed point of $\T^{1\mu}$, which in turn corresponds to
$\q^\kappa$ (Eq.~\eqref{eq:qkappa}).

\subsection{The off-policy termination operator}
We are now ready to present the operator underlying the Q($\beta$)
algorithm, which is the main contribution of this paper.
Eq.~\eqref{eq:tb-full} can be considered a special case where we correct {\em all} of the off-policy-ness,
thus implicitly assuming $\beta = {\mathbf 1}$. To obtain the general
case, we need to split the target in each TD-error into two off- and
on-policy terms:
\begin{align}
\R^\mu_\beta q(s, o) & = q(s,o) + \sum_{t=0}^\infty 
 \gamma^t \E_\pio\Big[\Big(\prod_{i=1}^t c^o_i \Big)
                       \delta^{\beta,\mu}_t \Big], \label{eq:general-update} \\
  \delta^{\beta,\mu}_t & = r(S_t, A_t) + \gamma\tilde{q}(S_{t+1}, o) -
    \q(S_{t}, o), \notag \\
\tilde{q}(s, o)  & = (1-\beta^o(s))\q(s, o) +
                             \beta^o(s) \E_\mu \q(s,\cdot), \notag \\
c^o_i & = (1  - \zeta^o(S_i)) \left( 1-\beta^o(S_i) +
                           \beta^o(S_i) \mu(o|S_i) \right). \notag
 \end{align}
Inspecting these updates, it is clear that the parameter $\beta$
controls the degree of partial ``off-policy-ness'': in the next state,
with a weight
$1-\beta^o(S_{t+1})$ the option continues and the update considers the
``on-policy'' current option value $\q(S_{t+1},o)$, and with a weight
$\beta^o(S_{t+1})$, the option terminates, and the update considers the ``off-policy''
value w.r.t. the policy over options $\sum_{o'}\mu(o'|s)\Q(S_{t+1},o')$. Note
that, these coefficients appear in the correction terms $c^o_i$ as well, since
in order for the algorithm to converge to the correct on-/off-policy
mixed target, the $\beta^o(S_{t+1})$-weighted off-policy portion needs to be
corrected. Finally, note that in the learning setting, the second factor in $c^o_i$ is explicit, while $1  -
\zeta^o(S_i)$ is sampled from the current option. 
Algorithm~\ref{alg:qbeta} presents the
forward view of the algorithm underlying this expected operator, for
the general case of an evolving sequence of policies $(\mu_k)_{k\in\bN}$.

This algorithm is very similar to the recently formalized
Q($\sigma$)~\cite{deasis2017multistep,sutton-barto17}, with $\beta$ being
a state-option generalization of $1-\sigma$. In Q($\sigma$), the
parameter $\sigma$
controls the degree of off-policy-ness: $\sigma=1$ corresponds
to (on-policy) SARSA, and $\sigma=0$ to (off-policy)
Tree-Backup. Analogously, Q($\beta$) with
$\beta={\mathbf 0}$ learns the ``on-policy'' value of the
current option policy $\pio$ 
(i.e. the policy $\iota$), and Q($\beta$) with $\beta={\mathbf 1}$
learns the ``off-policy'' value of
the marginal policy $\kappa$ (i.e. the policy $\mu$).
The behavior terminations $\zeta$ on the other hand have a role
analogous to that of the eligibility trace
parameter $\lambda$. 

 \begin{algorithm}[t]
 \caption{Q($\beta$) algorithm}
\begin{algorithmic}[1]
\medskip
\item[\textbf{Given:}] Option set $\O$, target termination function $\beta$,
initial Q-function $q_0$, step-sizes $(\alpha_k)_{k \in \bN}$, start
state $s_0$
 \STATE $S_0 \gets s_0$
\FOR{$k = 0, 1, \ldots$}
   \STATE Sample an option $o$ from $\mu_k(\cdot|S_0)$
   \STATE Sample the return  $R_1, S_1, R_2 \dots, S_{D_k}$ from
   $\pio$. $D_k$ is determined by sampling $1 - \zeta^o(S_i)$.
    \FOR{$t = 0, 1, \dots D_k-1$}
           \STATE $\delta^{\beta,\mu_k}_t = R_{t+1} +
             \gamma\tilde{q}_{\mu_k} (S_{t+1}, o) - \q(S_{t}, o)$
           \STATE $\tilde{q}_{\mu_k}(s, o)  \defequal (1-\beta^o(s))\q(s, o) +
          \beta^o(s) \E_{\mu_k} \q(s,\cdot)$
           \STATE $c^o_j = 1-\beta^o(S_j) + \beta^o(S_j) \mu(o|S_j) $
           \STATE $\Delta_t =
          \sum_{i=t}^{D_k-1}\gamma^{i-t}\Big(\prod_{j=t+1}^i c^o_j\Big)
          \delta^{\beta,\mu_k}_t $ 
         \STATE $\q_{k+1}(S_t, o) \gets \q_k(S_t, o) + \alpha_k\Delta_t $
    \ENDFOR
    \STATE $S_0 \gets S_{D_k}$
  \ENDFOR
\end{algorithmic}
\label{alg:qbeta}
\end{algorithm}

\subsection{Relationship with intra-option learning} 
The off-policy-ness discussed so far
is 
different than that in the more familiar off-policy intra-option
setting. The intra-option learning algorithm suggests applying the
update~\eqref{eq:intra-option}
to all options $o$ ``consistent with'' the experience
stream $S_1, A_1, \ldots$~\cite{sutton1998intra}. 
When considering multiple steps,
this amounts to introducing importance sampling into Eq.~\eqref{eq:intra-option-multi-step}.
That is, given a behavior option $b$, the trace coefficient $1-\bo(S_i)$ from
Eq.~\eqref{eq:intra-option-multi-step} now becomes 
\begin{equation}
  c^{ob}_i = \frac{\pio(A_i|S_i)}{\pi^b(A_i|S_i)} (1 - \bo(S_i)). \label{eq:intra-option-is}
\end{equation}
And hence, writing $\bo_{t+1}$ for $\bo(S_{t+1})$, the value of
option $o$ from Eq.~\eqref{eq:intra-option-multi-step} writes:
\begin{align*}
  q(s, o) & = \E_{\pi^b} \sum_{t=0}^\infty 
 \gamma^t \Big[\Big(\prod_{i=1}^t c^{ob}_i \Big) [R_{t+1} + \gamma\bo_{t+1}\E_\mu\q(S_{t+1},\cdot)] \Big]
\end{align*}
Now, notice that there are two sources of off-policy-ness in these
formulas. One is $\pi^b$ vs. $\pi^o$, the contrast between option
policies, and the other is in the target:  $\iota$ 
vs. $\mu$ itself, as discussed before. Indeed if we write the above in the form from
Eq.~\eqref{eq:peng} we get a different correction:
\begin{align}
\Delta q(s, o) & = r^{\pi_o}(s) + \gamma\E_\mu\q(S_{t+1},\cdot) -
                 \q(s, o) \notag  \\
& \quad + \E_{\pi^b}\Big[ \sum_{t=1}^\infty 
 \gamma^t \Big(\prod_{i=1}^{t-1} c^{ob}_i\Big) (1 -
                 \bo(S_t)) \delta^{ob}_t\Big], \notag \\ \label{eq:intra-option} 
  \delta^{ob}_t & = \frac{\pio(A_t|S_t)}{\pi^b(A_t|S_t)}
    [R_{t+1} +  \gamma\E_\mu\q(S_{t+1},\cdot)] -\E_\mu\q(S_{t},\cdot). \notag
\end{align}
Since the corrections for the two sources of off-policy-ness are orthogonal,
it could be possible to combine them. 
We leave this for the future.

\section{Analysis}
In this section we will analyze the convergence behavior of the
off-policy termination operator in both policy evaluation and control settings,
and show that learning about shorter target options off-policy is generally
asymptotically more efficient than on-policy.  We will then consider
the relationship of the solution quality in control with option
duration, and show that shorter options generally yield better solutions.

\subsection{Policy evaluation}

 We will prove that the evaluation
operator $\R^\mu_\beta$  is contractive
around the appropriate fixed point, and that its contraction factor is
less than that of the respective {\em on-}policy operator, given the
target options are longer than the behavior ones.

\begin{theorem}[Policy evaluation]
\label{thm:convergence-prediction}
The operator $\R^\mu_\beta$ defined in Eq.~\eqref{eq:general-update} has a
unique fixed point $\Q^{\mu,\iota}_{\beta}$, as defined in Eq.~\eqref{eq:q-mu-iota}. 
Furthermore, 
if for each state $S_i\in\S$ and option $o\in \O$ 
 we have $c^o_i \leq (1-\beta^o(S_i)) + \beta^o(S_i) \mu(o|S_i)$,
then for any Q-function $\Q$:
\begin{equation*}
| \R^\mu_\beta \Q(s,o) - \Q^{\mu,\iota}_{\beta}(s,o) | \leq \eta(s,o) \| \Q(s,o) - \Q^{\mu,\iota}_{\beta}(s,o) \|,
\end{equation*}
where $\eta(s,o) \defequal 1 - (1 - \gamma) \E_\pio \Big  [ \sum_{t=0}^\infty
  \gamma^t \Big(\prod_{i=1}^t c^o_i\Big) \Big ] \leq \gamma.$
\end{theorem}
\begin{proof}
The proof is analogous to that of Theorem 1 from \cite{munos2016safe}
and is given in appendix.
\end{proof}

\begin{figure*}[h] 
  \centering
  \includegraphics[scale=0.35]{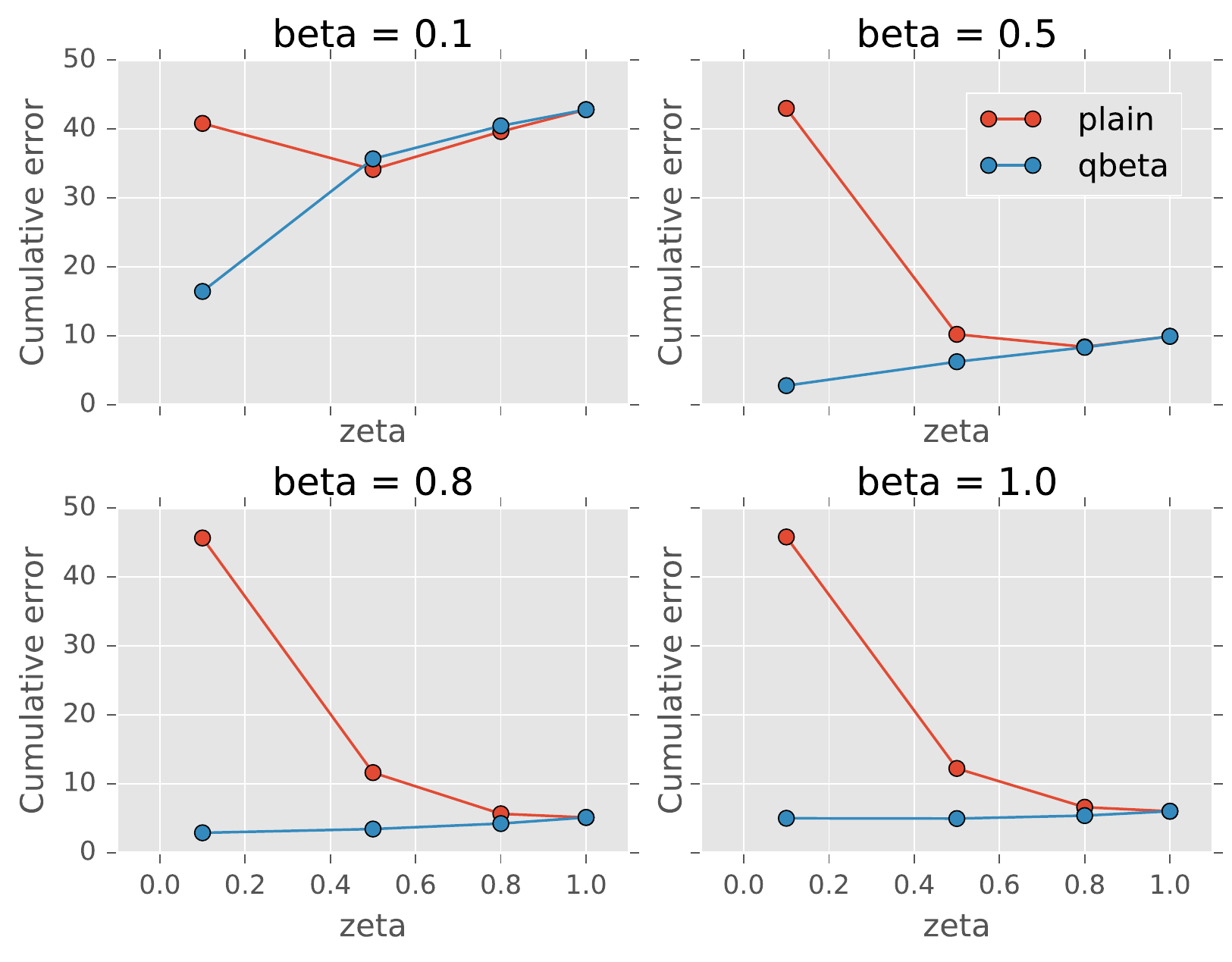}
\includegraphics[scale=0.35]{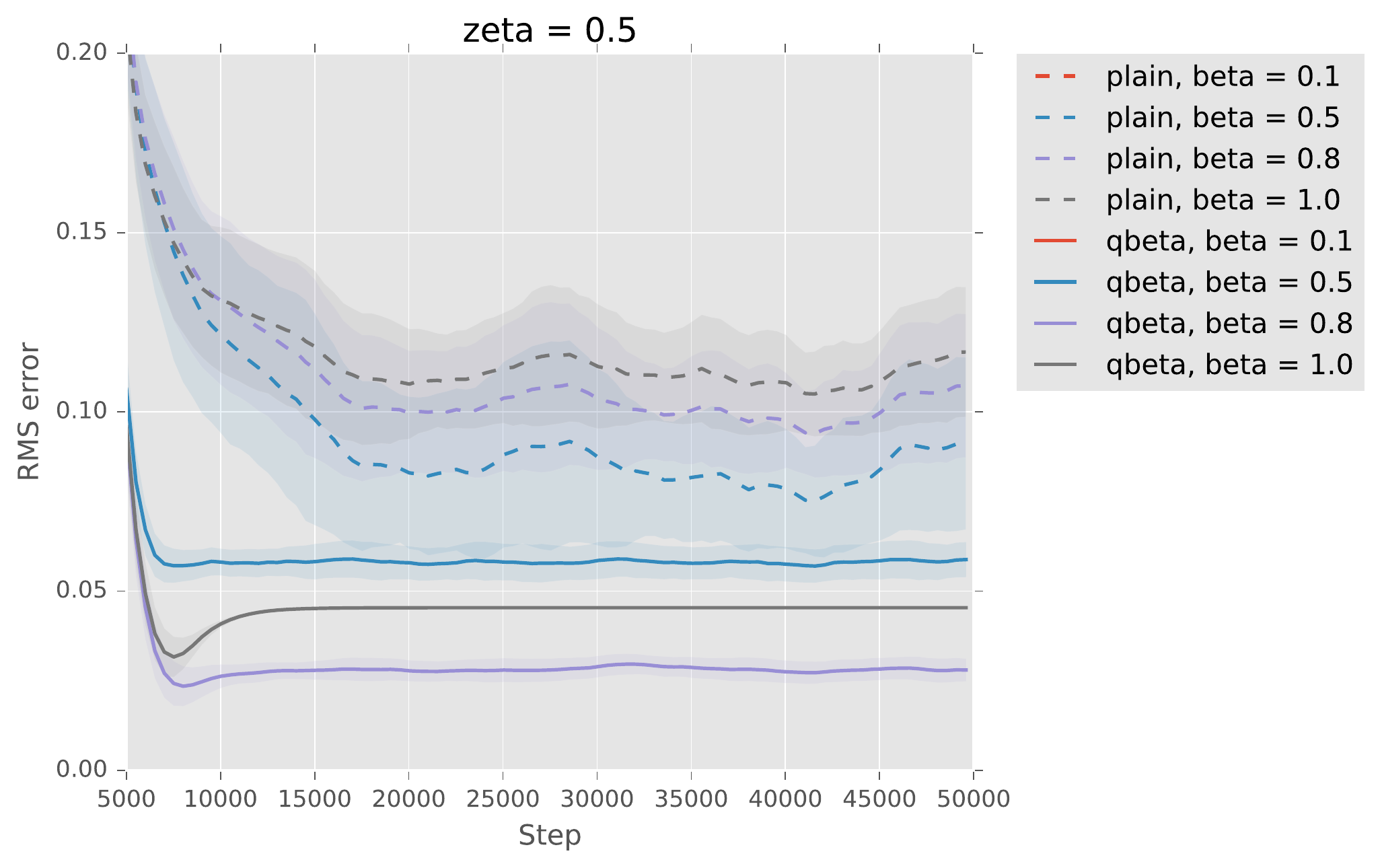} 
 \caption{Prediction error on the 19-state chain task. Each variant is
   an average of 10 seeds. {\bf Left:}
  Sum error for each $\zeta$-$\beta$ combination. Q($\beta$) always gets more efficient as $\zeta$ decreases (the
   options get longer).
{\bf Right:} Example
   learning curves. The lines corresponding to $\beta=0.1$ are outside
   the axes' bounds. The shaded region covers standard deviation.}
  \label{fig:prediction}
\end{figure*}

The contraction coefficient $\eta$ controls the convergence speed of
this operator: the smaller $\eta$ (and the larger $\prod_{i=1}^t
c^o_i$) the fewer iterations are needed to converge, but the larger
the computational expense when planning, or 
the variance when
learning~\cite{bertsekas1996temporal,munos2016safe}. 
Since $c^o_i \leq 1$, and options terminate eventually, the variance
is less significant here, and generally, larger $c^o_i$ will yield
faster convergence. 
In our case, since a behavior option is assumed (i.e. the
$1-\zeta^o(S_i)$ factor in Eq.~\eqref{eq:general-update} is fixed), the additional
$\beta$-term in $c^o_i$ can only {\em reduce} the existing trace.
However, since we are interested in learning about a
different target, we ought to compare these traces to the setting in
which that same
target is learnt {\em on-}policy. The following corollary derives a sufficient
condition for Q($\beta$) to maintain larger traces than its on-policy counterpart.
\begin{corollary}
\label{corr:faster-convergence}
  If for all states $s$ and options $o$ we have:
\[\beta^o(s) \geq 1-\frac{\mu(o|s)  (1-\zeta^o(s))}{\mu(o|s) (1 - \zeta^o(s)) +
            \zeta^o(s)}, \] then the
        iteration~\eqref{eq:general-update} converges faster than if
        $\beta^o$ was used on-policy, in place of $\zeta^o$ in iteration~\eqref{eq:intra-option-multi-step}.
\end{corollary}
In particular if $\zeta^o(s) = 0$, any $\beta^o(s) > 0$ satisfies
this, irrespective of $\mu$. If $\mu$ is deterministic, this holds for
all $\beta^o(s) > \zeta^o(s)$ for the chosen $o$. In general, the intuition here is that
it's easier to learn from longer option traces about shorter options
than vice versa. 

\subsection{Control}
\newcommand{\qmi}{\q^{\mu,\iota}}

Let us now formulate the control analogue of
Theorem~\ref{thm:convergence-prediction}. Its proof is a simpler version of that of
Theorem 2 from \cite{munos2016safe}, and we omit it here.



\begin{theorem}[Control]
\label{thm:control}
Consider a sequence of policies over options
$(\mu_k)_{k\in\bN}$ that are greedy w.r.t.~the sequence of estimates
$(\Q_k)_{k\in\bN}$, and consider the update: \[\Q_{k+1} = \R^{\mu_k}_{\beta} \Q_k,\]
where the operator $\R^{\mu_k}_{\beta}$ is defined by
Eq.~\eqref{eq:general-update} for the $k$-th policy $\mu_k$.
Let $\T^{\mu,\iota}_\beta \q = \T^{(1-\beta)\iota}\q +
\T^{\beta\mu}\q$, and let $\q^* =
\max_{\mu}\q^{\mu,\iota}_\beta$. Suppose that $\T^{\mu_0,
  \iota}_\beta\Q_0\geq \Q_0$.
\footnote{
This can  be attained by pessimistic initialization:
$\q_0 \leq \r_{\max} /
  (1-\gamma).$ }
Then for any $k\geq 0$,
\begin{equation*}
\| \Q_{k+1} - \q^* \| \leq \gamma\| \Q_k -  \q^* \|.
\end{equation*}
It follows that $\q_k\rightarrow \q^*$ as $t\rightarrow \infty$.
\end{theorem}

We do not give a proof of online convergence at this time, but 
verify it empirically in our experiments.




\subsection{Option duration and solution quality}
Our convergence results show that learning about shorter options
off-policy is more efficient than on-policy. Here, we motivate why one
would want to learn about shorter options in the first place. In
particular, we will show that given a set of options, the more they
terminate, the better the resulting control solution at the primitive
action resolution. Intuitively, this is because upon termination, the learner picks the
 current best option, whereas during option execution, the target value includes the potentially suboptimal
current option (Proposition~\ref{prop:offpol-options}). The following
theorem (proof in appendix) formalizes this intuition, and may be of
independent interest.
\begin{theorem}[The more options terminate, the better the solution.]
\label{thm:marginal-upper-bound}
 Given a set of options $\O$, and a greedy
 policy over options $\mu$. Let $\beta \geq \zeta$ be two
 termination conditions for the options in $\O$. Then: $\qmi_\beta \geq \qmi_\zeta.$
\end{theorem}
Note that this result refers to the target solution. During learning,
the more decisions there is to make, the more
potential there is for error. As such, in reasonably complex tasks we expect the performance to obey a tradeoff on $\beta$.

\section{Experiments}


Finally, let us evaluate our algorithm empirically.
We aim to illustrate the following claims:
\begin{itemize}
 \item The learning speed improves as $\zeta$ gets smaller (behavior options
   get longer);
\item The control performance improves, as $\beta$ gets
  larger (target options get shorter);
 \item Q($\beta$) converges with off-policy terminations.
\end{itemize}

For simplicity, we assume 
that  options terminate deterministically in a set of goal states. We
hence reduce $\beta$ and $\zeta$ to single parameters that determine the
likelihood of terminating {\em before} reaching the goal. The ``plain''
variant refers to the on-policy intra-option
update from Eq.~\eqref{eq:intra-option-multi-step}. The complete experimental
details are provided in appendix~\ref{sec:exp-details}.

\begin{figure}[h]
  \centering
   \includegraphics[scale=0.9]{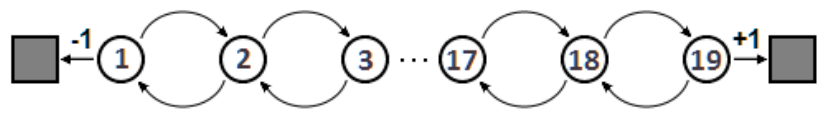}
  \caption{The 19-state random walk task. The agent starts in the
    middle. Transitions are
    deterministic, and the task terminates in each end.
  }
  \label{fig:chain}
\end{figure}

\subsection{Policy evaluation}


First, we show that Q($\beta$) learns the correct values on the
19-state random walk task (Fig.~\ref{fig:chain}).
There are two options: one leads all the way to the left, the other to the right.
 The policy over options is uniform. 
The task is to estimate the value function w.r.t. target terminations $\beta$.
The results are given in Figure~\ref{fig:prediction}. Q($\beta$) is
able to learn the correct values (up to an irreducible
exploration-related error). As expected, Q($\beta$) gets more
efficient as behavior options get longer ($\zeta$ gets smaller). The opposite is true for the plain algorithm,
   since without interrupting sufficiently, it does not have a chance
   to update the intermediate states with anything other than the
   option policy. There is a small inflection point in the performance
   of the plain algorithm at the {\em on-}policy value of $\zeta$. 

\begin{figure*}
    \centering
    \begin{subfigure}{0.56\textwidth} 
        \centering
        \includegraphics[scale=0.34]{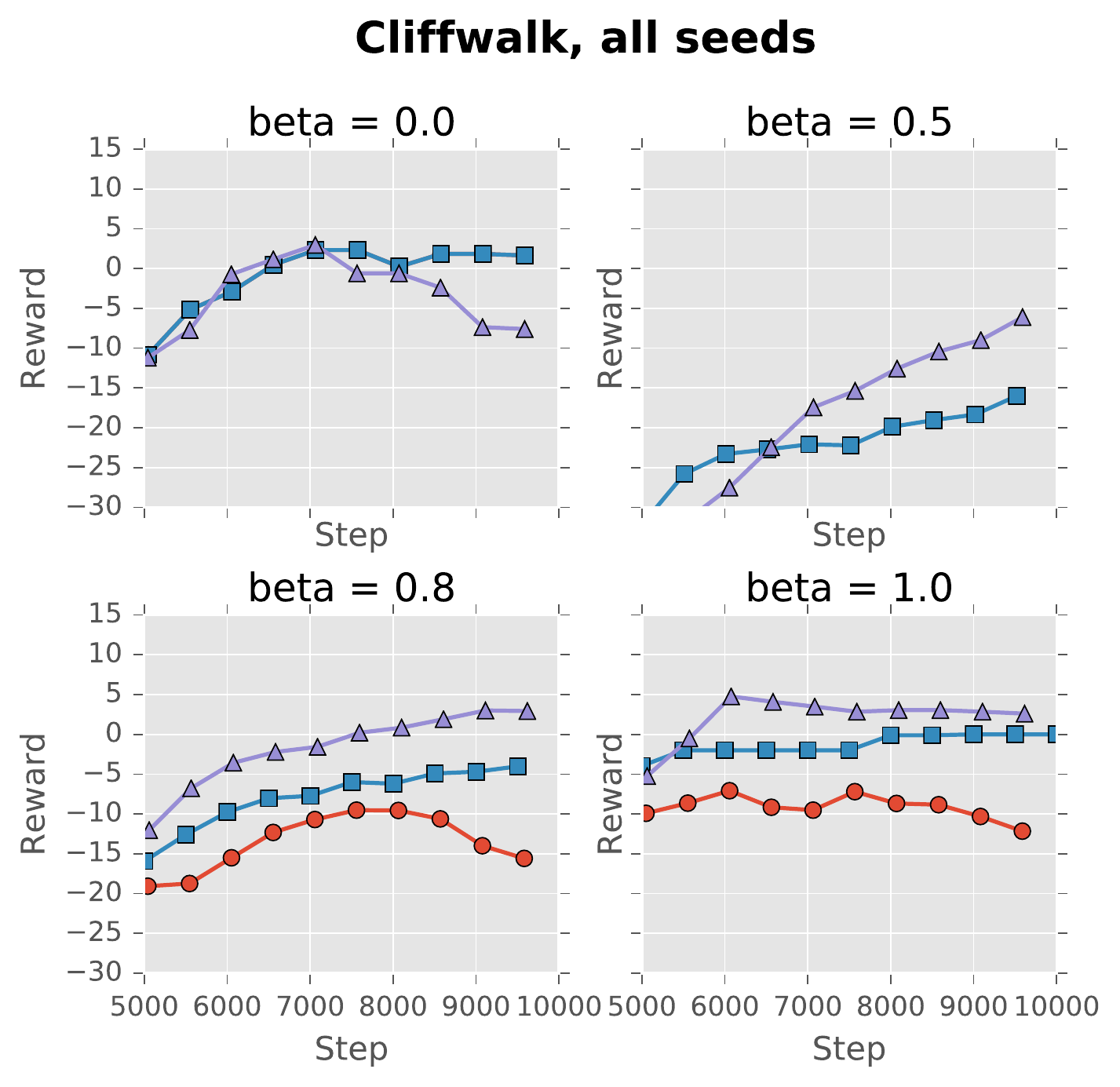}
       \includegraphics[scale=0.34]{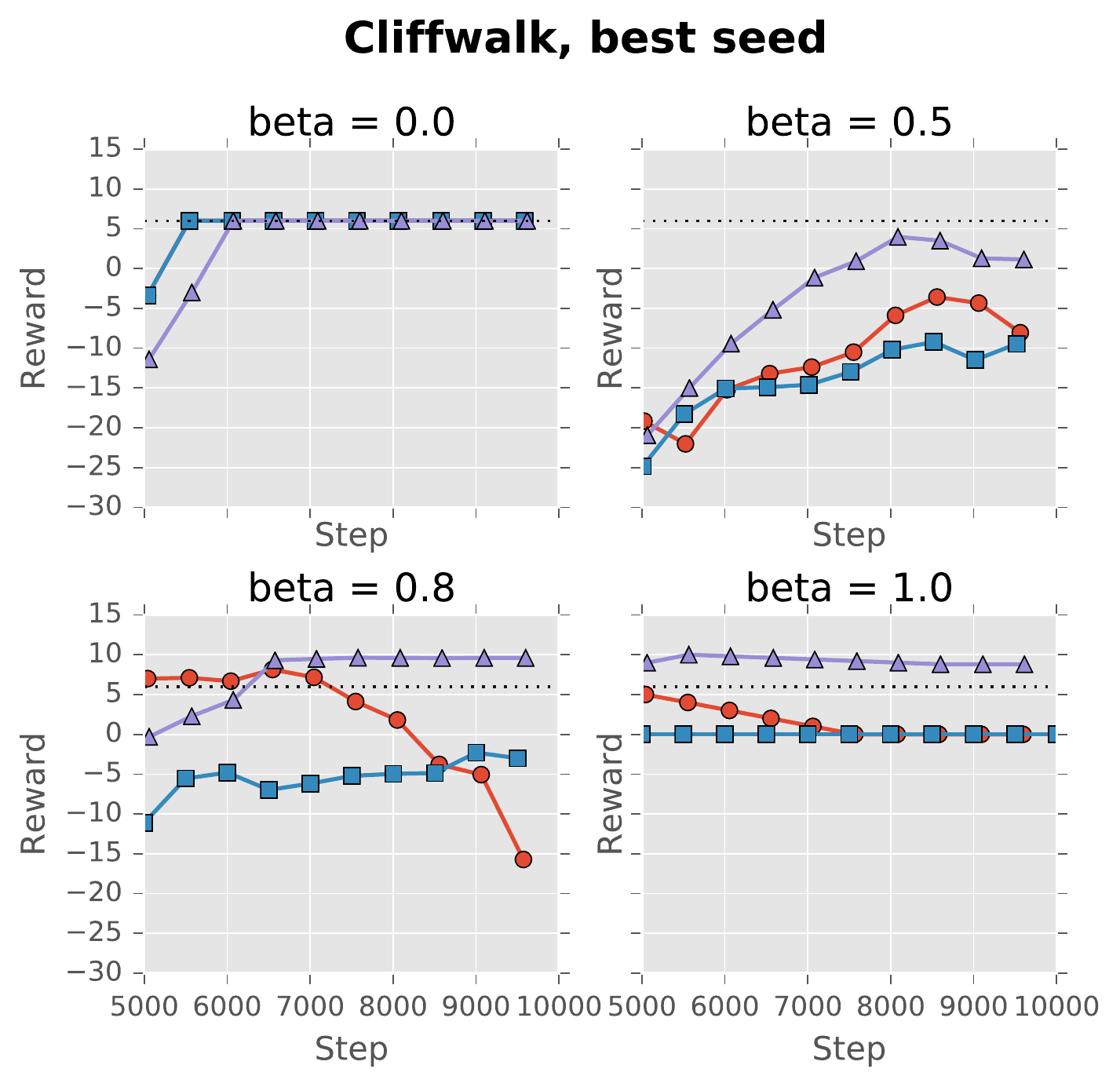}\\
       \includegraphics[scale=0.4]{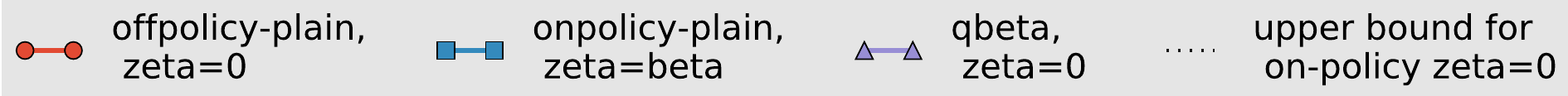}
        \caption{}\label{fig:toy-control}
   \end{subfigure}
\hfill
    \begin{subfigure}{0.43\textwidth} 
      \centering
        \includegraphics[scale=0.34]{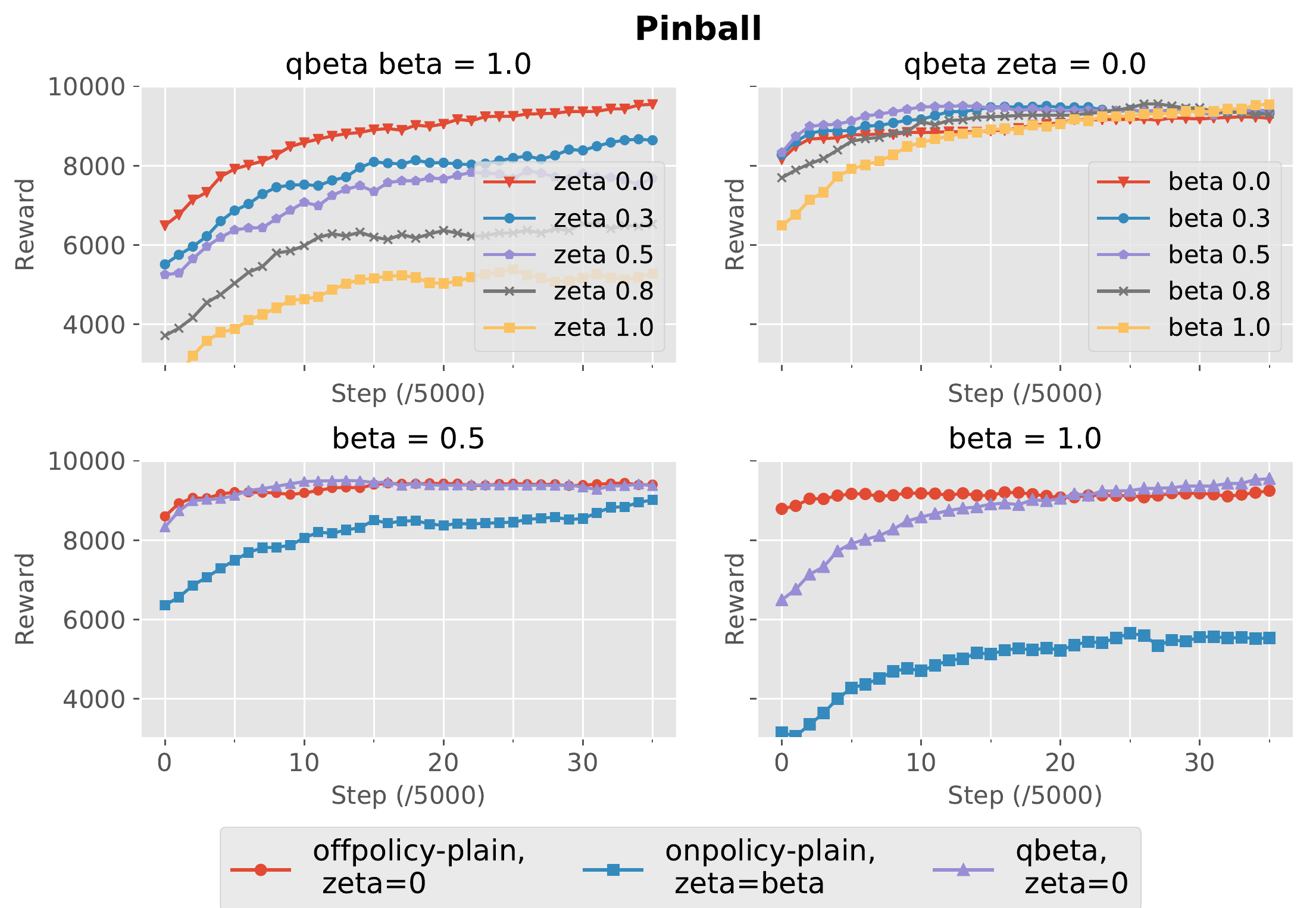}
        \caption{}\label{fig:pinball}
    \end{subfigure}
    \caption{Control performance. {\bf (a)} Cliffwalk. Each variant is
      evaluated on 5 seeds for 10 runs each. {\em Left:} Average performance per value
          of $\beta$ on all seeds.  {\em Right: } Learning curves for
          the best seeds per variant. Notice how Q($\beta$) is the
          only variant that escapes the plateau of the suboptimal
          policy. {\bf (b)} Pinball. Each variant is evaluated on 20
    independent runs.  {\em Top row:} Influence of $\beta$ and $\zeta$ on
    Q($\beta$): performance improves as $\zeta$ gets smaller. Overall,
    intermediate target $\beta$-s are best, but $\beta=1$ reaches
    slightly better performance at the end of learning.
{\em Bottom row:} Comparison within
   the variants for select values of $\beta$}
\end{figure*}

\begin{figure}
  \centering
\includegraphics[scale=0.32]{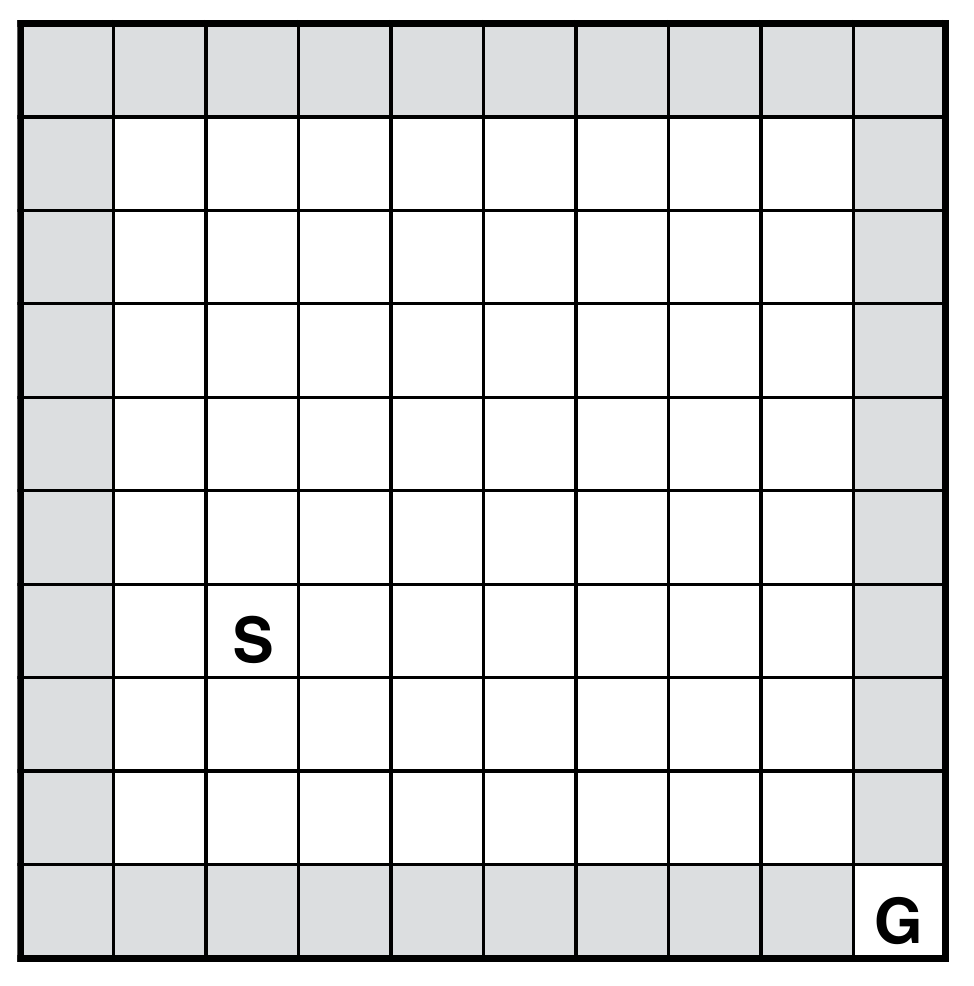}
\hspace{2em}
  \includegraphics[scale=0.2]{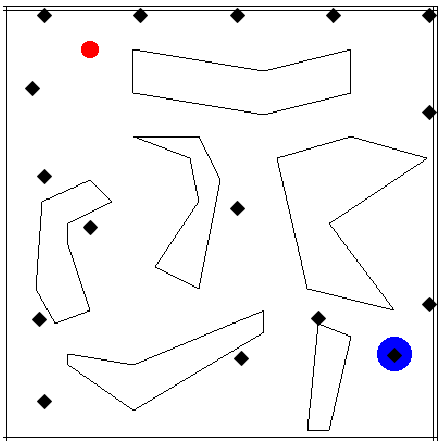}
\caption{
{\bf Left:} The modified cliffwalk task. Shaded regions are cliffs.
{\bf Right:}
Pinball domain configuration
    used. The red ball must be moved to the blue hole. Each black
    diamond indicates an option landmark.}
\label{fig:domains}
\end{figure}

\subsection{Control}

To demonstrate the benefit of decoupled off- and on- policy
terminations, we compare our algorithm with the plain on-policy variant (labelled: {\em onpolicy-plain}) that uses $\zeta=\beta$ during both learning and
evaluation. In order to demonstrate that the learning target plays a
role, we also compare it with the plain algorithm that uses $\beta$ at
evaluation only (labelled: {\em offpolicy-plain}). The behavior
$\zeta=0$, unless specified
otherwise.

\subsubsection{Modified Cliffwalk}
We illustrate the benefits of off-policy termination on a modified
Cliffwalk example (Fig.~\ref{fig:domains}, left).
The agent starts in a position
inside a $n\times n$ grid with the goal of getting to a corner where a
positive reward is given.  The step reward is zero, but there are
small cliffs along the border that aren't fatal, but induce a
penalty. We have four options, one for each
cardinal direction, that take the agent up until the corresponding
border (and cliff). Thus, while these options are able to learn to reach the
goal in an optimal number of {\em steps}, they are unable to learn the
optimal policy which only moves inside the grid, so as to not
encounter cliffs. To ensure adequate exploration we consider
$\epsilon_{opt}$-soft option policies, as well as a usual
$\epsilon$-greedy policy over options during learning (but not during
evaluation). 
The results are plotted in Fig.~\ref{fig:toy-control}. Q($\beta$)
outperforms the alternatives in all cases, and its performance
improves with larger $\beta$.
Note that it is the only variant
able to surpass the value of the suboptimal upper bound of the naive approach.

\subsubsection{Pinball}
We finally evaluate our algorithm on a variation of the Pinball
domain~\cite{KonidarisSkillChaining09}. 
Here, a small ball must be maneuvered through a set of obstacles into
a hole. Observations consist of four continuous variables describing the
ball's $x, y$ positions and velocities. There are five primitive actions:
the first four apply a small force to either the $x$ or $y$ velocity,
the final action leaves all velocities unchanged. There is a step
penalty and a final reward. We define a set of {\em landmark}
options~\cite{mann2015approximate} that move the ball near a target
goal location on the board. The agent can initiate and terminate
each option from within some initiation and termination distances
from the respective landmark. To illustrate the benefits of terminating
suboptimal options, landmarks were placed in such a way that the paths
from start to the goal via the landmarks are suboptimal
(Fig.~\ref{fig:domains}, right).
The results are plotted in Fig.~\ref{fig:pinball}.
As expected, the
performance of Q($\beta$) drastically improves with longer behavior
options. The influence of $\beta$ is more subtle: intermediate
$\beta$-s perform best overall, but $\beta=1$ reaches slightly 
better eventual performance.

In a comparison, Q($\beta$) outperforms the on-policy
variant that learns with $\zeta=\beta$. However, in this domain, the
off-policy variant (that learns with $\zeta=0$, but evaluates with
$\beta$) performs comparably to Q($\beta$).
This may in part be due to
the use of function approximation, which allows the
plain update to generalize meaningfully within the option trajectory, and in part
due to the noisy nature of Pinball, in which there are many optimal
policies of similar values. Since, as we have seen, Q($\beta$) is the
only variant to learn accurate values, we expect it to stand out more in settings where the reward scheme is more intricate.



\section{Related work}

Much of the related work has already been discussed throughout.
We mention a few more
relevant works below.

\citeauthor{mann2014time}~\shortcite{mann2014time} propose an algorithm for multi-step option
interruption that stems from the same motivation of mitigating
poor-quality options. In order to avoid the resulting options being too short, they introduce a time-regularization term. Our approach
bypasses the need to do so by interrupting off-policy and the
ability to explicitly specify target
terminations. 

\citeauthor{ryan2002using}~\shortcite{ryan2002using}
also considers the problem of {\em termination improvement}, or
interrupting options when they are no longer relevant, in a hybrid
decision-theoretic and classical planning setting. The intuitions in
that work are particularly aligned with the ones here -- in particular the tradeoff between efficiency and
optimality is hinted at, and it is even suggested that ``persistence [with an option] is a kind of
off-policy exploration, at the behavior level.'' 
In this work, we formalize and exploit this intuition.

\citeauthor{yu2012weighted}~\shortcite{yu2012weighted} consider general $\lambda$-operators with
state-dependent $\lambda$.
In the case of options, $1-\zeta^o$ takes the role of a
state-option-dependent
$\lambda$. \citeauthor{white2017unifying}~\shortcite{white2017unifying}
proposes to
consider $1-\zeta^o$ as part of the transition-based {\em discount} instead.

\section{Discussion}

We propose decoupling behavior and target termination
conditions, like it is done with policies in off-policy learning. We
formulate an algorithm for learning target terminations off-policy,
analyze its expected convergence, and validate it empirically,
confirming the theoretical intuition that learning shorter options
from longer options is beneficial both computationally and
qualitatively. 
More
generally, we cast learning with options into a common framework with well-studied multi-step off-policy temporal difference
  learning, which allows us to carry over existing results with ease.

\subsubsection{Learning longer options from shorter options.} We have
assumed here that the options are given, but may not
express the optimal policy well. This scenario applies when the
options describe simple rules of thumb, or are transferred from a different task. If the options are not given, but learnt
end-to-end, our wish typically is to distill meaningful behavior in
them. However, instead, the result often ends up reducing to degenerate
options~\cite{bacon2016option,mann2014time}. Being able to impose
longer durations on the target off-policy may mitigate this. 
Though it should be noted that our convergence results suggest that
learning may not be as efficient then.


\subsubsection{Action-level importance sampling.}
The option policy term in the trace is ambivalent to the
action choice. Thus if $\mu(o|S_i)$ is small, $c^o_i$ will be small, even if the taken
action is consistent with the option policy $\pio$. 
It would be interesting to replace this term with the importance
sampling ratio at the primitive action level, like
$\frac{\kappa(A_i|S_i)}{\pi^o(A_i|S_i)}$, which corresponds to
another multi-step off-policy algorithm,
Retrace($\lambda$)~\cite{munos2016safe}. 
Another direction towards
this goal is to incorporate the intra-option
correction from Eq.~\eqref{eq:intra-option-is}.

\subsubsection{Online convergence.}  
Proving online convergence of Q($\beta$) remains an open
problem. 
Supported by reliable empirical behavior, we hypothesize that it holds under
reasonable conditions and plan to investigate it further in the future.

\subsubsection{Future.} 
A natural direction for future work is to learn
many $\beta$-s in parallel as is done in classical off-policy
learning~\cite{sutton11}, or learn a new kind\footnote{I.e.,
  not one indifferent to the reward, as in
  \cite{szepesvari2014universal}, but one indifferent to the termination scheme.} of a {\em universal}
option model using the ideas from \cite{schaul2015universal}. It is
also worthwhile to extend the convergence results to approximate
state spaces, perhaps similarly to \cite{touati2017convergent}. 

\bibliographystyle{aaai}
\bibliography{bibliography}

\appendix

\section{Proof of Proposition~\ref{prop:offpol-options}}
\begin{proof}
From Eq.~\eqref{eq:To}, and writing $\q$ for $\Q^{\mu,\iota}_\beta$ for
less clutter:
  \begin{align*}
    \q & = (I - \gamma\P^{(1-\beta)\iota})^{-1} (\r^\pi + \gamma
     \P^{\beta\mu} \q), \\ 
  \q- \gamma\P^{(1-\beta)\iota} \q& = \r^\pi + \gamma
     \P^{\beta\mu} \q, \\ 
  \q & = \r^\pi+ \gamma \P^{\beta\mu}\q + \gamma\P^{(1-\beta)\iota} \q \\
  & = \r^\pi + \gamma \P^{\beta\mu}
                    \q+ \gamma\P^{1\iota} \q-
                    \gamma \P^{\beta\iota}\q \\
   & = \r^\pi + \gamma (\P^{\beta\mu} - \P^{\beta\iota})\q  +
     \gamma\P^{1\iota} \q.
\end{align*}
Solving for $\q$ yields the result. 
\end{proof}

\section{Proof of Theorem~\ref{thm:convergence-prediction}}
\begin{proof}
\newcommand{\qtarg}{\q^{\mu,\iota}_\beta}
\newcommand{\btarg}{\beta^o} 
Write $\R$ for $\R^\mu_\beta$.
 The fact that the fixed point is the desired one is clear  from
  Eq.~\eqref{eq:general-update} and
  Proposition~\ref{prop:offpol-options}. Recall that:
\begin{align*}
c^o_i & = 
  (1-\beta^o(S_i) + \beta^o(S_i) \mu(o|S_i)) (1 -
                           \zeta^o(S_i)), \\
\tilde{q}(s, o)  & = (1-\beta^o(s))\q(s, o) +
                             \beta^o(s) \E_\mu \q(s,\cdot). \\
 \end{align*}
Now, let us derive the contraction result. 
Eq.~\eqref{eq:general-update} can be written:
\begin{align*}
  \R \q(s,o) & =
                 \sum_{t=0}^\infty\gamma^t \E_{\pio}\Big[\Big(\prod_{i=1}^tc^o_i\Big)
               T^{co}_t \Big] \\
  T^{co}_t & \defequal r(S_t,A_t)  + \gamma(\tilde{q}(S_{t+1},o) - c^o_{t+1}\q(S_{t+1},o))
\end{align*} 
Let $\Delta\q(s,o) \defequal \q(s, o) - \qtarg(s, o)$, and
$\Delta\q(s,\cdot)$ be defined analogously. 
Since $\R\qtarg = \qtarg$, and after shifting the sum index forward, we have:
\begin{align*}
  \R \q(s,o) - \qtarg(s,o) =
  \sum_{t=1}^\infty\gamma^t \E_{\pio}\Big[\Big(\prod_{i=1}^{t-1}c^o_i\Big)
  \Delta T^{co}_t \Big], 
\end{align*}
\begin{align*}
   \Delta T^{co}_t & =  (1-\beta^o_t)\Delta\q(S_t,o) 
+ \beta^o_t\E_\mu\Delta\q(S_t,\cdot) -
    c^o_{t}\Delta\q(S_t,o)\\
 & =(1-\beta^o_t-c^o_t)\Delta\q(S_t,o) 
+ \btarg_t\sum_{b\in\O}\mu(b|S_t)\Delta\q(S_t,b) \\
  & = \sum_b\Big( \1_{b=o} (1-\beta^b_t - c^b_t) + \beta^o_t\mu(b|S_t)
    \Big) \Delta\q(S_t,b),
\end{align*}
where we write $\1_{b=o}$ for the indicator function. 
Since $c^o_{t} \leq (1-\btarg_t) + \btarg_t \mu(o|S_t)$ and $\btarg_t
\mu(b|S_t) \geq 0$, $\forall b\neq o$, we have a linear combination of $\Delta\q(S_t,b)$
weighted by non-negative coefficients $w_{y,b}$ defined as:
$w_{y,b} =   \sum_{t=1}^\infty\gamma^t \E_{\pio}\Big[\Big(\prod_{i=1}^{t-1}c^o_i\Big)
          (\1_{b=o} (1-\beta^b_t - c^b_t) + \beta^o_t\mu(b|S_t))
          \1_{S_t = y}\Big],$
whose sum is: 
\begin{align*}
  \sum_{b\in\O} w_{y,b} & = \sum_{t=1}^\infty\gamma^t \E_{\pio}\Big[\Big(\prod_{i=1}^{t-1}c^o_i\Big)
          \sum_{b\in\O} \1_{b=o} (1-\beta^b_t - c^b_t) \\
  & \qquad \qquad \qquad + \beta^o_t \mu(b|S_t)
                     \Big] \\
   & = \sum_{t=1}^\infty\gamma^t \E_{\pio}\Big[\Big(\prod_{i=1}^{t-1}c^o_i\Big)
     ((1-\btarg_t - c^o_t) \\
  & \qquad\qquad \qquad + \sum_{b\in\O} \mu(b|S_t) \beta^o_t \Big] \\
  & = \sum_{t=1}^\infty\gamma^t \E_{\pio}\Big[\Big(\prod_{i=1}^{t-1}c^o_i\Big)
    ((1-\btarg_t - c^o_t) + \beta^o_t )\Big] \\
  & = \sum_{t=1}^\infty\gamma^t \E_{\pio}\Big[\Big(\prod_{i=1}^{t-1}c^o_i\Big) (1-
    c^o_t)\Big] \\
   & = \E_{\pio} \left[ \sum_{t=1}^\infty \gamma^t
     \Big(\prod_{i=1}^{t-1}c^o_i\Big) -  \sum_{t=1}^\infty \gamma^t
     \Big(\prod_{i=1}^{t}c^o_i\Big) \right] \\
  & = \gamma C - (C - 1) \leq \gamma
\end{align*}
where $C = \E_{\pio}\left[\sum_{t=0}^\infty \gamma^t
     \Big(\prod_{i=1}^{t}c^o_i\Big)\right]$, and the last inequality
   is due to $C \geq 1$. It follows that $\R$ is a
   $\gamma$-contraction around $\qtarg$.
\end{proof}

\section{Proof of Corollary~\ref{corr:faster-convergence}}
\begin{proof}
We would like for the off-policy trace $c^o_s =  (1-\beta^o_s + \beta^o_s \mu(o|s))  (1 - \zeta^o_s)$
to be larger than the equivalent on-policy trace $1 -
\beta^o_s$.
\begin{align*}
  1 - \beta^o_s & \leq (1 - \beta^o_s + \beta^o_s \mu(o|s))  (1 - \zeta^o_s) \\
 (1 - \beta^o_s) \zeta^o_s  & \leq \beta^o_s \mu(o|s) (1 -
                                    \zeta^o_s)  \\
\zeta^o_s & \leq \beta^o_s (\mu(o|s) (1 - \zeta^o_s) + \zeta^o_s) \\
\beta^o_s & \geq \frac{\zeta^o_s}{\mu(o|s) (1 - \zeta^o_s) +
            \zeta^o_s} \\
 & = 1-\frac{\mu(o|s)  (1-\zeta^o_s)}{\mu(o|s) (1 - \zeta^o_s) +
            \zeta^o_s} 
\end{align*}
Thus, if $\beta^o$ obeys this bound, it is more beneficial to learn it off-
rather than on- policy, from the point of view of convergence speed.
\end{proof}

\section{Proof of Theorem~\ref{thm:marginal-upper-bound}}
\begin{proof}
First notice that Proposition~\ref{prop:offpol-options} rewrites:
\begin{align}
   \q & = (I -  \gamma (\P^{\beta\mu} - \P^{\beta\iota})  -
     \gamma\P^{1\iota} )^{-1} \r^\pi \notag \\    
& = (I -  \gamma \P^{\beta\mu} - \gamma \P^{(1-\beta)\iota})^{-1}
  \r^\pi \notag \\
 & = \sum_{t=0}^\infty\gamma^t(\P^{\beta\mu} + \P^{(1-\beta)\iota})^t
   \r^\pi. \label{eq:prop-rewrite}
  \end{align}
  Let $\Delta^\zeta_\beta \defequal \q^{\mu,\iota}_\beta  -
  \q^{\mu,\iota}_\zeta$. From the above~\eqref{eq:prop-rewrite}, we have:
  \begin{align*}
   \Delta^\zeta_\beta  & = \sum_{t=1}^\infty\gamma^t ((\P^{\beta\mu} +
     \P^{(1-\beta)\iota})^t - (\P^{\zeta\mu} + \P^{(1-\zeta)\iota})^t)
     \r^\pi \\
    & = (\P^{\beta\mu} + \P^{(1-\beta)\iota} - \P^{\zeta\mu} - \P^{(1-\zeta)\iota}) \sum_{t=1}^\infty\gamma^t {\cal A}_t
     \r^\pi,
\end{align*}
where ${\cal A}_t =  (\P^{\beta\mu} + \P^{(1-\beta)\iota} - \P^{\zeta\mu} -
\P^{(1-\zeta)\iota})^{-1}  ((\P^{\beta\mu} + \P^{(1-\beta)\iota})^t -
(\P^{\zeta\mu} - \P^{(1-\zeta)\iota})^t)$. 
\newcommand{\f}{f}
Let $f =\sum_{t=1}^\infty\gamma^t {\cal A}_t
     \r^\pi$ be a Q-function. Then simple manipulations yield:
\begin{align*}
    \Delta^\zeta_\beta = \q^{\mu,\iota}_\beta - \q^{\mu,\iota}_\zeta 
& = (\P^{(\beta-\zeta)\mu} + \P^{(I - \beta - I + \zeta))\iota} ) \f \\
    & = (\P^{(\beta-\zeta)\mu} - \P^{(\beta-\zeta)\iota} ) \f \\
    & = (\P^{1\mu}- \P^{1\iota}) (\beta-\zeta) f \\
    & \geq (\P^{1\mu}- \P^{1\iota}) \f \geq {\mathbf 0},
  \end{align*}
since the operators $\P^{1\mu}$ and $\P^{1\iota}$ are monotone, $\beta
\geq \zeta$ and $\P^{1\mu}\f = \max_\nu\P^{1\nu}\f \geq
\P^{1\iota}\f$ for any Q-function $f$.
\end{proof}

\section{Experimental details}
\label{sec:exp-details}

The setting is as follows: an option $o$ is picked according to $\mu$, and
a trajectory $s, R_1, S_2, R_2, \ldots, R_{D-1}, S_D$ is generated
according to $\pio$ and $\zeta^o$. Then for each state $S_i$ in the
trajectory, $\q(S_i, o)$ is updated according to the considered algorithm.

\subsection{19-chain}
The termination conditions $\beta$ and $\zeta$ are
evaluated in the range of $\{0.1, 0.5, 0.8, 1\}$, with the first value
being positive to ensure adequate state visitation. The step-sizes set via a linear search over
$\alpha\in\{0.1, 0.2, 0.3, 0.4\}$. The discount factor
$\gamma=0.99$.

\subsection{Modified Cliffwalk}
The reward scheme is: $r_{goal} = 10$ and $r_{cliff} = -2$, and the
grid size $n=10$.
We set $\epsilon=0.1$, and $\epsilon_{opt}=0.3$, and determine the step-size for
each variant from a linear search over $\alpha\in\{0.1, 0.2, 0.3,
0.4\}$, behavior $\zeta=0$, and target $\beta$ is evaluated on the range of $\{0, 0.5, 0.8,
1\}$.  The discount factor
$\gamma=0.99$.

\subsection{Pinball}
The reward is $-1$ on every step, except the final step which receives a
reward of $10000$. We use initiation distance of $0.3$ and termination
distance of $0.03$. The state option value function was approximated using tile coding
with $16$, $10\times 10$ tilings. All algorithms used a learning rate
$\alpha=0.01$, discount $\gamma=0.99$ and an exploration rate
$\epsilon=0.05$ and $\epsilon_{opt}=0.01$ during learning. The target
$\beta$ is evaluated on the range of $\{0, 0.3, 0.5, 0.8, 1\}$. 

\end{document}